\documentclass{sig-alternate}
\pdfoutput=1
\usepackage{subfigure}

\newcommand{\mytilde}{\raise.17ex\hbox{$\scriptstyle\mathtt{\sim}$}}

\newcommand{\xv}{\mathbf{x}}
\newcommand{\Xv}{\mathbf{X}}

\newcommand{\yv}{\mathbf{y}}

\newcommand{\wv}{\mathbf{w}}

\newcommand{\diag}{\mathrm{diag}}

\newcommand{\muv}{\boldsymbol \mu}

\newcommand{\xiv}{\boldsymbol \xi}
\newcommand{\alphav}{\boldsymbol \alpha}

\newcommand{\omegav}{\boldsymbol \omega}
\newcommand{\gammav}{\boldsymbol \gamma}

\newcommand{\Sigmav}{\boldsymbol \Sigma }

\newcommand{\N}{\mathcal{N}}

\newcommand{\data}{\mathcal{D}}

\newcommand{\argmax}{\operatornamewithlimits{argmax}}

\newtheorem{theorem}{Theorem}

\newtheorem{lemma}[theorem]{Lemma}

\newcommand{\myfigure}[2]{
\begin{figure}
\centering
\epsfig{file=#2, trim=0 25 0 55, clip=true, width=1.0\linewidth, height=0.7\linewidth}
\caption{#1}
\label{fig:#2}
\end{figure}
}

\newcommand{\myfigref}[1] {\ref{fig:#1}}

\newcommand{\tabref}[1] {\ref{tab:#1}}

\newcommand{\myitemize}[1]{\begin{itemize}#1\end{itemize}}

\begin{document}

\title{Fast Parallel SVM using Data Augmentation}

\numberofauthors{1}

\author{
\alignauthor
Hugh Perkins*, Minjie Xu*, Jun Zhu, and Bo Zhang \\
\affaddr{Department of Computer Science}\\
\affaddr{Tsinghua University}\\
\affaddr{Beijing, 100084 China} \\
\affaddr{
hughperkins@gmail.com,
chokkyvista06@gmail.com,\\
dcszj@mail.tsinghua.edu.cn,
dcszb@mail.tsinghua.edu.cn}
}



\newcommand{\fix}{\marginpar{FIX}}
\newcommand{\new}{\marginpar{NEW}}

\maketitle

\renewcommand{\thefootnote}{\fnsymbol{footnote}}
\footnotetext{* denotes equal contribution.}

\begin{abstract}
As one of the most popular classifiers, linear SVMs still have challenges in dealing with very large-scale problems, even though linear or sub-linear algorithms have been developed recently on single machines. Parallel computing methods have been developed for learning large-scale SVMs. However, existing methods rely on solving local sub-optimization problems. In this paper, we develop a novel parallel algorithm for learning large-scale linear SVM. Our approach is based on a data augmentation equivalent formulation, which casts the problem of learning SVM as a Bayesian inference problem, for which we can develop very efficient parallel sampling methods. We provide empirical results for this parallel sampling SVM, and provide extensions for SVR, non-linear kernels, and provide a 
parallel implementation of the Crammer and Singer model.  This approach is very promising in its own right, and further is a very useful technique to parallelize a broader family of general maximum-margin models.

\end{abstract}

\section{Introduction}\label{section:introduction}

Support vector machines (SVMs) are among the the most popular and successful paradigms to build classifiers. SVMs have demonstrated tremendous success in many real world applications. However, learning SVMs is a challenging problem. Traditional decomposition methods, like SVMLight~\cite{joachims1999making}, LibSVM~\cite{Chang:libsvm01} and SMO~\cite{Platt:98}, have cubic time complexity. The need for developing highly efficient algorithms has increased, due to the fact that large corpora are very easy to obtain, like the various Challenges on image categorization, object detection, document categorization, etc.

As the computing resources get cheaper, multi-core and multi-machine computing systems are not rare. For instance, it is not uncommon for a research group to build a computing system with hundreds of CPU cores. To harness the power of large clusters of computers, developing the distributed algorithms for SVMs has received a lot of attention. Representative works include the parallel SVM (PSVM)~\cite{Chang:PSVM07}, which performs approximate matrix factorization to reduce memory use and then uses the interior point method to solve the quadratic optimization problem on multiple machines in parallel. The parallel mixture method~\cite{Collobert:02} and the cascade SVM~\cite{Graf:nips04} decompose the entire learning problem into multiple smaller QP problems and solve them in parallel. 

Recently, the frequency of CPU cores has reached a point where increasing the frequency further is
not cost-effective, because of the increase in power consumption.  Modern hardware contains
an increasing number of low-power cores, epitomized by the recent growth of GPGPU hardware.  Thus, on future hardware, the fastest algorithm might not be the one that runs fastest in a single thread, but the one which can run effectively on parallel hardware.

In this paper, we present a very simple and highly efficient distributed algorithm for learning SVMs. Our algorithm is built on the recent work~\cite{Polson:BA11}, which shows that the learning problems of SVM can be equivalently formulated as hierarchical Bayesian model, with additional scale variables. Based on the hierarchical formulation, we can develop Monte Carlo methods to infer the parameters (or their posterior distributions). More importantly, the sampling algorithm can be easily parallelized.

Our work is also inspired by the recent developments on distributed Monte Carlo methods for improving the scalability of probabilistic latent topic models~\cite{Smola:vldb10,newman2009distributed}.

Our parallel method is interesting in its own right, because it can be massively and scalably parallelized. In our experiments, we showed scalability up to 500 cores for large datasets.  Not only does parallelizing allow one to take advantage of the distributed processing in commodity clusters, but also the large amount of distributed memory, so it is possible to run on huge datasets which is otherwise even impossible to be loaded into memory on single machines.

In addition, it is a useful addition to our armory, because it can be used to solve composite models, such as MedLDA~\cite{Zhu:icml13}, without needing to make the mean-field assumption.  There are many models, such as \cite{xu:nips12} for example, that may be able to benefit from fast and accurate parallelization using the parallel sampling or parallel EM SVM formulation.

We have extended Polson's formulations to provide formulations in addition for support vector regression (SVR), non-linear kernelized SVM, and the Crammer and Singer multiclass model.

We provide parallel implementations for a linear SVM, a non-linear kernelized SVM, a formulation for SVR, and a parallel solver for the
Crammer and Singer multiclass SVM model.

\paragraph{Outline} Section \ref{section:RTM} reviews the formulation of an SVM as a Bayesian inference problem, Section \ref{section:extensions} extends
the linear sampling SVM to non-linear kernels, to regression, and to the Crammer and Singer multiclass model.  Section \ref{section:parallelsvm} presents
the use of the sampling SVM to implement a parallel, distributed SVM.  Lastly, Section \ref{section:experiments} presents experiments comparing our
parallel SVM implementation with recent state of the art SVM solvers.

We show that the parallel linear SVM can give excellent performance on very large datasets, where the number of samples is large in comparison to the square of the number of features, and there is parallel hardware available.  In these cases, we show that our implementation can give training times faster than other state of the art linear solvers, such as StreamSVM, whilst giving comparable accuracy.

\section{SVM as Bayesian Inference}\label{section:RTM}

In this section we present the fundamental theories on which our extensions and distributed algorithms are built.

\subsection{SVM: the Basics}

We first focus on standard linear SVMs for binary classification. Let $\data = \{(\xv_d, y_d)\}_{d=1}^D$ be the training data, where $y_d \in \{1, -1\}$. The goal of SVMs is to learn a linear discriminant function
    $$f(\xv; \wv, \nu) = \wv^\top \xv + \nu.$$
For notation simplicity, we absorb the offset parameter $\nu$ into $\wv$ by introducing an additional feature dimension with fixed unit value. To find the optimal $\wv$, the canonical learning problem of SVMs with a tolerance on training errors is formulated as a constrained optimization problem
\setlength\arraycolsep{1pt}
\begin{eqnarray}
    \min_{\wv, \xiv} && \frac{1}{2} \lambda \Vert \wv \Vert_2^2 + 2 \sum_d \xi_d \nonumber \\
    \mathrm{s.t.}:&&\forall d,~\left\{
    \begin{array}{l}
        y_d \wv^\top \xv_d \geq 1 - \xi_d \\
        \xi_d \geq 0
    \end{array}\right. ,  \nonumber
\end{eqnarray}
Note that the constant factor $2$ in the training error term can be absorbed into $\lambda$, yet we leave it for the simplicity of the deduction later.
Slack variables removed, the problem is equivalently formulated as an unconstrained form
\begin{eqnarray}\label{eq:SVM}
    \min_{\wv} \frac{1}{2} \lambda \Vert \wv \Vert_2^2 + 2 \sum_d \max(0, 1- y_d \wv^\top \xv_d ),
\end{eqnarray}
which is known as the regularized risk minimization framework. For binary classification, the loss is called hinge loss.

\subsection{SVM: the MAP estimate}

Problem~\eqref{eq:SVM} can also be viewed as a MAP estimate of a probabilistic model, where the posterior distribution is
$$p(\wv | \data) \propto q_0(\wv) q(\yv | \wv, \Xv),$$
where $q_0(\wv) = \mathcal{N}(\mathbf{0}, \lambda^{-1} I)$ and $q(\yv | \wv, \Xv) = \prod_d q(y_d | \wv, \xv_d)$ with
\begin{eqnarray}
    q(y_d | \wv, \xv_d) = \exp( - 2\max(0, 1-y_d \wv^\top \xv_d) ).
\end{eqnarray}
Note that we factorize the posterior into $q_0$ and $q$ merely for the simplicity of subsequent denotation and they normally are intrinsically different from the \emph{genuine} prior and likelihood as can be induced from the probabilistic model (even up to a constant factor). Hence we call $q_0$ and $q$ \emph{pseudo-}prior and \emph{pseudo-}likelihood respectively.

The benefit of the MAP formulation is that it allows us to take advantage of many existing techniques developed for inference in probabilistic models and hence grants more flexibility for the solution. Specifically, Polson and Scott~\cite{Polson:BA11} show that the pseudo-likelihood can be represented as a scale mixture of Gaussians, namely
\begin{lemma}Scale mixture for hinge loss
    \begin{align}
        & \exp( - 2\max(0, 1-y_d \wv^\top \xv_d) ) \nonumber \\
        = & \int_0^\infty \frac{1}{ \sqrt{ 2\pi \gamma_d} } \exp\Big( -\frac{ (1 + \gamma_d - y_d \wv^\top \xv_d)^2 }{2 \gamma_d}  \Big) d\gamma_d
    \end{align}
\end{lemma}
This directly inspires an augmented representation with $\gammav=(\gamma_1,\dots,\gamma_D)$ such that
\begin{align*}
    p(\wv,\gammav|\data) &\propto q_0(\wv)\prod_d q(y_d,\gamma_d|\wv,\xv_d)\\
    q(y_d|\wv,\xv_d)&=\int_0^\infty q(y_d,\gamma_d | \wv, \xv_d)d\gamma_d\\
    q(y_d,\gamma_d|\wv,\xv_d)&=\phi(1-y_d \wv^\top \xv_d|-\gamma_d,\gamma_d)
\end{align*}
where $\phi(\cdot|\mu,\sigma^2)$ is the Gaussian density function.

\subsection{MCMC Sampling for SVM}
Based on this augmented representation, we are able to design MCMC methods for $p(\wv,\gammav|\data)$, from which the optimal SVM solution that maximizes $p(\wv|\data)$ is relatively more probable to get sampled. 

Specifically, we use Gibbs sampling and have the following conditional distributions~\cite{Polson:BA11}
\begin{eqnarray}
    p(\wv | \gammav, \data) &=& \mathcal{N}(\muv, \Sigmav) \label{eq:SVMgamma} \\
    p(\gamma_d^{-1} | \wv, y_d,\xv_d) &=& \mathcal{IG}(|1 - y_d \wv^\top \xv_d|^{-1}, 1 )  \label{eq:SVMgamma2},
\end{eqnarray}
where
\begin{equation}\label{mu_sigma}
    \Sigmav = \Big(\lambda I + \sum_d \frac{1}{\gamma_d} \xv_d \xv_d^\top \Big)^{-1},~ \muv = \Sigmav \Big( \sum_d y_d(1 + \frac{1}{\gamma_d}) \xv_d \Big)
\end{equation}
and $\mathcal{IG}$ is the inverse Gaussian distribution.

\subsection{EM algorithm for SVM}
The EM algorithm is useful when directly maximizing the posterior $p(\wv|\data)$ is intractable but it's easy to alternate between the following two steps which converges to a local maximum of the posterior.
\begin{align}
    \text{E-step: }&Q^{(m)}(\wv)&=&\int\log p(\wv,\gammav|\data)p(\gammav|\data,\wv^{(m)})d\gammav\\
    \text{M-step: }&\wv^{(m+1)}&=&\argmax_{\wv}Q^{(m)}(\wv)
\end{align}
One can prove that the algorithm above monotonically increases the genuine posterior distribution of interest $p(\wv|\data)$ after each iteration, just as traditional EM does likelihood.

Deduction details omitted to save space, we summarize the results as follows
\begin{align}
    \text{E-step (update $\gammav$): }&&\gamma_d^{(m)}&=|1-y_d\wv^{(m)\top}\xv_d|\\
    \text{M-step (update $\wv$): }&&\wv^{(m+1)}&=\muv^{(m+1)}(\gammav^{(m)})
\end{align}
where $\muv$ is calculated just as Eq.~\eqref{mu_sigma}.


Although normally EM is not guaranteed to obtain the global optimum (even after infinite iterations), for our specific $p(\wv|\data)$ which is concave w.r.t $\wv$, global optimum is expected. Furthermore, EM is a deterministic algorithm and enjoys a straightforward stopping criterion when compared with MCMC sampling.

\section{Extensions}\label{section:extensions}

In this section we extend the idea above to SVR, nonlinear kernel SVMs, and the Crammer and Singer multi-class SVM.

\subsection{Learning Nonlinear Kernel SVMs}

According to the representer theorem, the solution to problem~(\ref{eq:SVM}) has the form
\begin{eqnarray}
\wv = \sum_d \alpha_d y_d \xv_d,
\end{eqnarray}
which is a linear combination of $\Xv$. We can naturally extend it to the nonlinear case by using a feature mapping function $h$ and learn the nonlinear SVM by solving
 \begin{eqnarray}\label{eq:KernelSVM}
    \min_{\wv} \frac{1}{2} \lambda \Vert \wv \Vert_2^2 + 2 \sum_d \max(0, 1- y_d \wv^\top h(\xv_d) ),
\end{eqnarray}
whose solution can be represented accordingly as
\begin{eqnarray}\label{eq:repnonlinear}
    \wv = \sum_d \alpha_d y_d h(\xv_d) = H \diag(\yv) \alphav,
\end{eqnarray}
where $H = [h(\xv_1)~h(\xv_2)~\cdots~h(\xv_D)]$.

Substituting Eq.~\eqref{eq:repnonlinear} into~\eqref{eq:KernelSVM}, we get the dual problem
\begin{align}\label{eq:DualKernelSVM}
    \min_{\alphav}~&\frac{1}{2} \lambda \alphav^\top \diag(\yv) K \diag(\yv) \alphav ~+\nonumber \\
    & 2 \sum_d \max(0, 1- y_d \alphav^\top \diag(\yv) K_d^\top ),
\end{align}
where $K$ is the Gram matrix and $K_d$ is the $d$th row. If the feature map function $h$ is a reproducing kernel, i.e., $h(\xv) = k(\cdot, \xv)$, problem~\eqref{eq:DualKernelSVM} becomes a kernel SVM and each entry of $K$ is a dot product, that is $$K_{ij} = k(\xv_i, \xv_j) = h(\xv_i)^\top h(\xv_j).$$
The Gram matrix $K$ is positive definite for any reproducing kernel, e.g. the most commonly used Gaussian kernel
$$k(\xv_i, \xv_j) = \exp\Big( - \frac{\Vert \xv_i - \xv_j \Vert_2^2}{2 \sigma^2} \Big)$$

Let $\omegav =  \diag(\yv) \alphav$, then $\wv = \sum_d \omega_d h(\xv_d)$ and the problem becomes
 \begin{eqnarray}\label{eq:DualKernelSVM2}
    \min_{\omegav} \frac{1}{2} \lambda \omegav^\top K \omegav + 2 \sum_d \max(0, 1- y_d \omegav^\top K_d^\top ),
\end{eqnarray}


Observing the similarity between problem~(\ref{eq:DualKernelSVM2}) and (\ref{eq:SVM}), we reformulate it as MAP just as we did (\ref{eq:SVM}), with $q_0(\omegav)=\mathcal{N}(0, (\lambda K)^{-1} )$ and $q(\yv | \omegav, \Xv) = \prod_d q(y_d | \omegav, \xv_d)$, where
\begin{equation}
    q(y_d | \omegav, \xv_d)=\exp( - 2\max(0, 1-y_d \omegav^\top K_d^\top) ).
\end{equation}

\begin{lemma}Scale mixture for kernel hinge loss
    \begin{align}
        & \exp( - 2\max(0, 1-y_d \omegav^\top K_d^\top) ) \nonumber \\
        = & \int_0^\infty \frac{1}{ \sqrt{ 2\pi \gamma_d} } \exp\Big( -\frac{ (1 + \gamma_d - y_d \omegav^\top K_d^\top)^2 }{2 \gamma_d}  \Big) d\gamma_d
    \end{align}
\end{lemma}

Consequently for kernel SVMs, we have
    \begin{eqnarray}
        q(\omegav | \gammav, \data) &=& \mathcal{N}(\muv, \Sigmav) \\
        p(\gamma_d^{-1} | \wv, y_d,\Xv) &=& \mathcal{IG}(|\ell - y_d \omegav^\top K_d^\top|^{-1}, 1 ),
    \end{eqnarray}
    where $$\Sigmav = \Big(\lambda K + \sum_d \frac{1}{\gamma_d} K_d^\top K_d \Big)^{-1},~ \muv = \Sigmav \Big( \sum_d y_d(1 + \frac{1}{\gamma_d}) K_d^\top \Big).$$

\subsection{Support Vector Regression}

For regression, where the response variable $y$ are real-valued, the support vector regression (SVR) problem is defined as minimizing a regularized $\epsilon$-insensitive loss~\cite{Smola:03}
\begin{eqnarray}\label{eq:SVR}
\min_{\wv} \frac{1}{2} \lambda \Vert \wv \Vert_2^2 + 2 \sum_d \max(0, |y_d - \wv^\top \xv_d| - \epsilon ),
\end{eqnarray}
where $\epsilon$ is the precision parameter\footnote{$\epsilon$ is a small positive number, e.g., $1e^{-3}$ in our experiments}.

Naturally, we obtain the same $q_0$ as SVMs and
\begin{eqnarray}
    q(y_d | \wv, \xv_d) = \exp( - 2\max(0, |y_d - \wv^\top \xv_d| - \epsilon) ),
\end{eqnarray}
and the augmentation is carried out by the following lemma
\begin{lemma}Double scale mixture for $\epsilon$-insensitive loss
    \begin{align}
        & \exp( - 2\max(0, |y_d - \wv^\top \xv_d|-\epsilon) ) \nonumber \\
        = &\int_0^\infty \frac{1}{ \sqrt{ 2\pi \gamma_d} } \exp\Big( -\frac{ (\gamma_d + y_d - \wv^\top \xv_d - \epsilon)^2 }{2 \gamma_d}  \Big) d\gamma_d \nonumber \\
        \times &\int_0^\infty \frac{1}{ \sqrt{ 2\pi \omega_d} } \exp\Big( -\frac{ (\omega_d - y_d + \wv^\top \xv_d - \epsilon)^2 }{2 \omega_d}  \Big) d\omega_d
    \end{align}
\end{lemma}
\begin{proof}
    As $\epsilon \geq 0$, the following equality holds
    {\small\begin{align}
        & \max(0, |y_d - \wv^\top \xv_d|-\epsilon) \nonumber \\
        = &\max(0, y_d - \wv^\top \xv_d - \epsilon) + \max(0, -y_d + \wv^\top \xv_d -\epsilon).
    \end{align}}%
    Therefore, for each term, we can do similar derivation as in Lemma 1 to get the double scale mixture formulation.
\end{proof}
Consequently for SVR, we have
\begin{eqnarray}
    p(\wv | \gammav, \omegav, \data) &=& \mathcal{N}(\muv, \Sigma) \\
    p(\gamma_d^{-1} | \wv, \omegav, y_d, \xv_d) &=& \mathcal{IG}(|y_d - \wv^\top \xv_d - \epsilon|^{-1}, 1 ) \label{eq:SVRgamma} \\
    p(\omega_d^{-1} | \wv, \gammav, y_d, \xv_d) &=& \mathcal{IG}(|y_d - \wv^\top \xv_d + \epsilon|^{-1}, 1 ), \label{eq:SVRomega}
\end{eqnarray}
where the covariance and mean are now
\begin{equation}
    \Sigmav = \Big( \lambda I + \sum_d(\frac{1}{\gamma_d} + \frac{1}{\omega_d}) \xv_d \xv_d^\top \Big)^{-1},
\end{equation}
\begin{equation}
    \muv = \Sigmav \Big( \sum_d( \frac{y_d - \epsilon}{\lambda_d} + \frac{y_d + \epsilon}{\omega_d}) \xv_d \Big).
\end{equation}

\subsection{Learning Multi-class SVM}

For multi-class classification, we have $y_d \in \{1, \cdots, M\}$. There are various strategies to perform multi-class classification with SVM. Here we consider the approach proposed by Crammer and Singer (2001), where the generalized discriminant function is defined to be
\begin{eqnarray}
    f(y, \xv; \wv) = \wv_y^\top \xv
\end{eqnarray}
where $\wv_y$ is the sub-vector corresponding to class label $y$. And the regularized risk minimization problem becomes
\begin{align}
    &\min_\wv \frac{1}{2}\lambda \Vert \wv \Vert_2^2 + 2 \sum_d \max_y(\Delta_d(y) - \Delta f_d(y;\wv) ),
\end{align}
where $\Delta_d(y)$ is the cost of predicting $y$ for the true label $y_d$ and $\Delta f_d(y;\wv) = f(y_d, \xv_d; \wv) - f(y, \xv_d; \wv)$ is the margin, and both $\Delta_d(y)$ and $\Delta f_d(y;\wv)$ equals zero when $y=y_d$.

Then, the pseudo-prior and pseudo-likelihood is changed accordingly to
{\small
\begin{align}
    q_0(\wv)&=\prod_y q_0(\wv_y)=\prod_y\N(\wv_y|\mathbf{0},\lambda^{-1}I)\\
    q(y_d|\wv,\xv_d)&=\exp(-2\max_y(\Delta_d(y) + \wv_y^\top \xv_d - \wv_{y_d}^\top \xv_d))
\end{align}
}%
In order for Lemma 1 to be applicable, we resort to an iterative procedure, which alternately infer weights $\wv_y$ given the other weights $\wv_{-y}$, for each class label $y$.

The local conditional distribution is
\begin{eqnarray}
    p(\wv_y | \data, \wv_{-y}) \propto q_0(\wv_y) \prod_d \psi(\wv_y;\wv_{-y},y_d,\xv_d),
\end{eqnarray}
where $\psi(\wv_y;\wv_{-y},y_d,\xv_d)\propto q(y_d|\wv,\xv_d)$
\begin{align}
&=\exp(-2(\max(\wv_y^\top\xv_d+\Delta_d(y),\zeta_d(y))-\wv_{y_d}^\top\xv_d))\quad\; \nonumber \\
&\propto\left\{\begin{array}{l}
        \exp(-2\max(\wv_y^\top\xv_d-\rho_d^y,0))\quad(y\neq y_d) \\
        \exp(-2\max(0,\rho_d^y-\wv_y^\top\xv_d))\quad(y=y_d) \\
    \end{array}\right.\\
&=\exp(-2\max(0,\beta_d^y(\rho^y_d-\wv^T_y \xv_d)))
\end{align}
where $\zeta_d(y) = \max_{y^\prime\neq y}(\wv_{y^\prime}^\top \xv_d+\Delta_d(y^\prime))$ is independent of $\wv_y$, $\rho_d^y=\zeta_d(y)-\Delta_d(y)$ and
$\beta_d^y = \begin{cases}+1 & \text{ for } y=y_d \\
                      -1 & \text{ for } y\neq y_d
\end{cases}$.

Hence we take
\[\psi(\wv_y;\wv_{-y},y_d,\xv_d)=\exp(-2\max(0,\beta_d^y(\rho^y_d-\wv^T_y \xv_d)))\]
and through a similar augmentation, we obtain the Gibbs sampling step for each augmented local conditional distribution $p(\wv_y,\gammav_y|\data,\wv_{-y})$
\begin{align}
        p(\gamma_{yd}^{-1} | \wv, y_d, \xv_d) &= \mathcal{IG}(|\rho^y_d - \wv_y^\top \xv_d|^{-1}, 1 ), \\
        p(\wv_y | \gammav_{y}, \wv_{-y}, \data) &= \mathcal{N}(\muv_y, \Sigmav_y) \label{eq:MC-SVMgamma}
\end{align}
where \begin{align} \Sigmav_y &= \Big(\lambda I + \sum_d \frac{1}{\gamma_{yd}} \xv_d \xv_d^\top \Big)^{-1}, \\
 \muv_y &= \Sigmav_y \Big( \sum_d ( \frac{\rho^y_{d}}{\gamma_{yd}} + \beta_d^y) \xv_d \Big). \end{align}
Note that this is actually a hierarchical Gibbs sampling
\begin{enumerate}
\item to sample $p(\wv|\data)$, we carry out Gibbs sampling over $p(\wv_y|\data,\wv_{-y})$ alternately for $y=1,\dots,M$;
\item to sample each $p(\wv_y|\data,\wv_{-y})$, we use data augmentation to sample over $p(\wv_y,\gammav_y|\data,\wv_{-y})$.
\end{enumerate}
Accordingly, the EM algorithm for Crammer and Singer multi-class SVMs inherits this 2-layer structure:
\begin{enumerate}
\item to maximize $p(\wv|\data)$, we carry out blockwise coordinate descent to maximize $p(\wv_y|\data,\wv_{-y})$ alternately;
\item to maximize each $p(\wv_y|\data,\wv_{-y})$, we adopt the EM algorithm where
\end{enumerate}
{\small\begin{equation*}
    Q^{(m)}(\wv_y)=\int\log p(\wv_y,\gammav_y|\data,\wv_{-y})p(\gammav_y|\data,\wv_y^{(m)},\wv_{-y})d\gammav_y
\end{equation*}}%

\section{Parallel SVM}\label{section:parallelsvm}

\begin{figure}
\centering
\epsfig{file=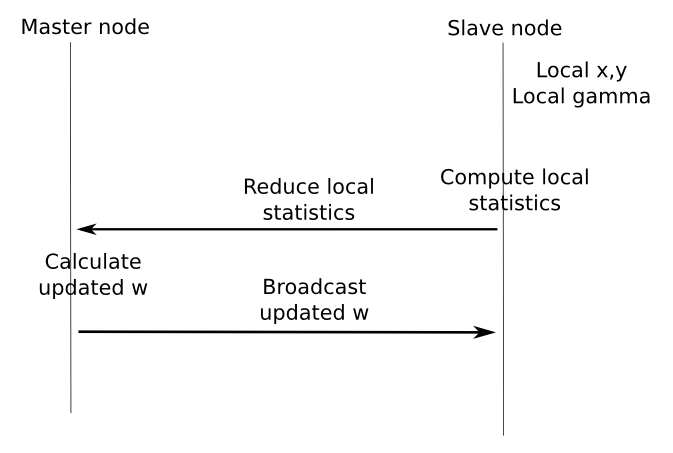, width=0.9\linewidth}
\caption{Map-reduce architecture for parallel sampling SVM}
\label{fig:mapreduce}
\end{figure}

Below we show how to employ distributed computing into the sampling algorithms above. We focus on the classical linear binary SVMs for the ease of explanation. And exactly the same techniques apply as well to all the extensions we present in section~\ref{section:extensions}, and also their EM algorithms.

Two key properties of the sampling process that are in favor of parallel computation are summarized as follows.
\begin{enumerate}
\item The scale variables $\gammav$ are mutually independent from each other, whose sampling step, therefore, can be easily parallelized to multiple cores and multiple machines.
\item The training data $(\xv_d,y_d)$ contribute to the global variables $\muv$ and $\Sigmav$ through a simple summation operator (Eq.~\eqref{mu_sigma}). Thus a typical map-reduce architecture is directly applicable, as shown in Figure \ref{fig:mapreduce}.
\end{enumerate}

\subsection{The Basic Procedure}

Let $P$ be the total number of processes and let $\data^p = \{(\xv_d^p, y_d^p)\}_{d=1}^{D_p}$ be the data assigned to process $p$. Then each process performs the following computations
\begin{enumerate}
    \item {\it draw scale parameters}: each $p$ draws $\gamma_{dp}^{-1}~(\forall 1 \leq d \leq D_p)$ according to the distribution in Eq. (\ref{eq:SVMgamma2}).
    \item {\it compute local statistics}: each $p$ computes the following local statistics
    \begin{eqnarray}
        ~~\muv^p &=& \sum_{d=1}^{D_p} (1 + \frac{1}{\gamma_{dp}})y^p_d \xv^p_d , \nonumber \\
            ~~\Sigmav^p &=& \sum_{d=1}^{D_p} \frac{1}{\gamma_{dp}} \xv^p_d \xv^{p\top}_d.
    \end{eqnarray}
\end{enumerate}
Since $\Sigmav^p$ is symmetric, it suffices to compute only the upper or lower triangle and then submit to the master.

After process $p$ has finished its local computation, it passes the local statistics $\muv^p$ and $\Sigmav^p$ to the master process, which collects the results and performs the following aggregation operations
\begin{enumerate}
    \item compute $\Sigmav^{-1} = \lambda I + \sum_p \Sigmav^p$.
    \item after $\Sigmav^{-1}$ is updated, compute $\muv = \Sigmav (\sum_p \muv^p)$.
\end{enumerate}

It is worth noting that all the slave processes perform exactly the same set of operations. Assume that we equally partition the large data set and all computing nodes are of the same capacity, then it can be expected that all the nodes have a high probability to finish their local job at roughly the same time. Therefore the latency due to synchronization is typically small. While in contrast, the existing parallel methods for SVMs by solving multiple smaller QP problems can suffer from large synchronization latency since the sub-QP problems varies a lot.

\subsection{Notation}

We will denote the parallel sampling SVM as PEMSVM.  PEMSVM has the following options:

\myitemize {
    \item linear (``LIN'') vs kernelized (``KRN'')
    \item EM (``EM'') vs MCMC (``MC'')
    \item binary classification (``CLS'') vs multiclass classification (``MLT'') vs support vector regression (``SVR'')
}

These three sets of options are orthogonal, so we can write a set of options for example as 'LIN-EM-CLS'.

$N$ is the number of training instances, $K$ is the number of features, $M$ is the number of classes, and $P$ is the number of processes.

\subsection{Iteration time}

We looked at the iteration time for different formulations, to give some indication of how well the implementation might scale with $N$, $K$, and $P$.

We found that all formulations are highly scalable in $P$.  The LIN formulation is very scalable in $N$, but finds datasets with high $K$ challenging. 
The calculations involve dense $K$ by $K$ matrices, even where the original data is sparse.  So, dense datasets will run relatively more quickly on our implementation than sparse ones, when comparing with other possible solvers.

By contrast, the KRN formulation is highly scalable in $K$, in fact, the iteration time is independent of $K$, but the iteration time is cubic with $N$,
which is a significant challenge.  To make the KRN formulation really effective, it might be useful to find some way to either reduce the number of
features, or use an approximation.  For example, PSVM approximates the $N$ by $N$ kernel matrix with an $N$ by $sqrt(N)$ matrix, and gets very good 
accuracy.  Maybe there is a way to do something similar with the sampling kernel SVM formulation?

As far as the Crammer and Singer solver, the scalability follows closely the scalability of the associated underlying solver, ie LIN, for the formulations we 
presented above.

Next we will present the reasons for the asymptotic iteration times we just talked about.

\subsubsection{EM}

\paragraph{LIN}

\begin{table}
\centering
\begin{tabular}{|l|l|}
\hline
Step & Asymptotic time \\
\hline
Draw $\boldsymbol \gamma$ & $O( NK / P )$ \\
Calculate $\boldsymbol \mu_p$ & $O( NK / P )$ \\
Calculate $\boldsymbol \Sigma_p$ & $O( N K^2 / P )$ \\
Reduce & $O(  K^2 \log(P) )$ \\
Draw $\boldsymbol \mu$ & $O( K^2 \log(K) )$ \\
Broadcast $\boldsymbol \mu$ & $O(  K^2 \log(P) )$ \\
\hline
\end{tabular}
\caption{Asymptotic times for LIN-EM-CLS.}
\label{tab:asymptoticlin}
\end{table}

Looking first at linear EM binary classification, LIN-EM-CLS comprises the steps shown in Table \ref{tab:asymptoticlin}.  Overall:
\begin{displaymath}
O( K^2[ N /P + \log(P) + \log(K) ] )
\end{displaymath}

Typically, the $N/P$ term dominates, giving $O(NK^2/P)$, and parallelization is effective.

Where $K$ or $P$ are high, then the $\log(P)$ and $\log(K)$ terms can dominate.  When this is the case,
(further) parallelization is no longer effective.

Therefore, parallelization is most effective for high $N$ and lower $K$.

\begin{table}
\centering
\begin{tabular}{|l|l|}
\hline
Step & Asymptotic time \\
\hline
Draw $\boldsymbol \gamma$ & $O( N^2 / P )$ \\
Calculate $\boldsymbol \mu_p$ & $O( N^2 / P )$ \\
Calculate $\boldsymbol \Sigma_p$ & $O( N^3 / P )$ \\
Reduce & $O(  N^2 \log(P) )$ \\
Draw $\boldsymbol \mu$ & $O( N^2 \log(N) )$ \\
Broadcast $\boldsymbol \mu$ & $O(  N^2 \log(P) )$ \\
\hline
\end{tabular}
\caption{Asymptotic times for KRN-EM-CLS.}
\label{tab:asymptotickrn}
\end{table}

\paragraph{KRN}
Next, turning to the kernel formulation for binary classification, KRN-EM-CLS comprises the steps in \tabref{asymptotickrn}.   Overall:
\begin{displaymath}
O( N^2[ N /P + \log(P) + \log(N) ] )
\end{displaymath}

Typically, the $N/P$ term dominates, giving $O(N^3/P)$, which shows effective parallelization.

When $P$ or $N$ are high, then the $\log(P)$ and $\log(N)$ terms can dominate, and (further) parallelization
is no longer effective.

Therefore, parallelization of the kernel formulation is most effective for high $K$ and low $N$.

\paragraph{SVR}

As far as SVR, the iteration time is asymptotically identical to CLS.  There is a constant factor of 2, but this is absorbed in asymptotic analysis.

\paragraph{MLT}

Looking at the Crammer and Singer solver formulation, the iteration time of MLT is multiplied by a factor of $M$, when compared to binary classification, CLS.

\subsubsection{MC}

The paragraphs above discussed the EM formulation.  In MC, there is an additional stochastic sampling step for both $\boldsymbol \gamma$ and
$\boldsymbol \Sigma$.  However, the asymptotic times of the sampling steps are no larger than other terms already considered, and the asymptotic
iteration time of MC is identical to that of EM.

\section{Experiments}\label{section:experiments}

We compare the parallel sampling SVM implementations with state of the art linear solvers.

\subsection{Summary of results}

We show that LIN-*-CLS is faster than state of the art linear solvers when $N$ is large relative to $K^2$, and there is parallel hardware available.

We show that LIN-*-MLT is highly scalable when $N$ is large relative to $K^2$, and there is parallel hardware available.

We show that the algorithms for KRN-*-CLS and LIN-*-SVR can give accuracy comparable to existing solvers.

We evaluate the performance of a GPU solver for LIN-EM-CLS. We show that the learning speed of our formulations can be accelerated by use of a GPU.

\subsection{Test conditions}

Tests were run on a cluster of 12-core nodes.  The cores were 2.6GHz; and each node had 24GB memory.

\subsection{Datasets}

\begin{table}
\centering
\begin{tabular}{|l|l|l|l|l|l|}
\hline
Name & N & K & M & Type & Source \\
\hline
alpha & 250,000 & 500 & 2 & CLS & Pascal LSL \\
dna & 25,000,000 & 800 & 2 & CLS & Pascal LSL\\
year & 250,000 & 90 & - & SVR & YearPredictionMSD \\
mnist8m & 4,000,000 & 798 & 10 & MLT & \\
\hline
\end{tabular}
\caption{Datasets}
\label{tab:datasets}
\end{table}

Table \ref{tab:datasets} shows the datasets used.

For some experiments, we created subsets.  A $K=K_0$ subset means that we include only features where $k <= K_0$.  Similarly a $N=N_0$ subset means that only the first $N_0$ data points from the original training dataset were included.

\subsection{Solvers}

\begin{table}
\centering
\begin{tabular}{|l|l|}
\hline
Name & Description \\
\hline
LL-Dual\cite{Fan:jmlr08} & Liblinear dual coordinate descent \\
 & L2-regularization L2-loss biased \\
LL-Primal\cite{Fan:jmlr08} & Liblinear primal coordinate descent \\
 & L2-regularization L2-loss biased \\
LL-CS\cite{Fan:jmlr08} & Liblinear Crammer and Singer \\
PSVM\cite{Chang:PSVM07} & PSVM, with rank\_ratio set to $1/\sqrt N$ \\
SVMPerf\cite{joachims2006training} & SVMPerf, with defaults \\
SVMMult\cite{joachims2009cutting} & SVMMulticlass, with defaults \\
Pegasos\cite{shalev2007pegasos} & Pegasos, with defaults \\
SDB\cite{chang2011selective} & Selective Block Minimization, biased \\
StreamSVM\cite{matsushima2012linear} & StreamSVM, with defaults \\
\hline
\end{tabular}
\caption{Solvers}
\label{tab:solvers}
\end{table}

Table \ref{tab:solvers} shows the solvers used, in addition to PEMSVM.

\subsection{Termination conditions}

PEMSVM calculates the value of the objective function at each iteration.  The algorithm terminates when the iterative change falls to $0.001 * N$ or below, which we found to be a reasonable stopping condition across many datasets.

We used the default termination conditions for other solvers.

\subsection{I/O}

By breaking the problem into independent sub-problems, not only can the calculations be
parallelized across multiple cores, but the I/O load of reading in the datafile into memory
can similarly be parallelized across cores, and across compute nodes.  This in itself can lead
to speed increases when compared to single-threaded algorithms.

In addition, even for large datasets, such as dna, it is possible to hold the dataset entirely in
memory, across multiple compute nodes.

\subsection{Implementation Details}

\subsubsection{MPI implementation}

MPI was used with C/C++ to parallelize the implementation over multiple CPU cores.  The
cores can be on a single node or multiple nodes.  Each MPI process was assigned
a partition of the dataset, read the data from the datafile itself, and coordinated with
a master process.

The MPI implementation was implemented using a sparse representation for $\xv_d$.

\subsubsection{GPU implementation}

OpenCL was used with C/C++ to parallelize the calculation of $\sum_d \frac{1}{\gamma_d}\xv_d \xv_d^T$ over multiple GPU cores.  The data was partitioned and each partition was loaded
into the local memory of a computer unit.  The results written to global memory, then
reduced using a second GPU kernel.

For multiple GPUs, the dataset was first partitioned, then each partition was handled by a single GPU, in parallel.

For datasets that did not fit into the GPU global memory, the dataset was first partitioned into
chunks that did, then each chunk was processed sequentially using the above algorithm.

The GPU implementation was implemented using a dense representation for $\xv_d$, for simplicity, though there is no technical reason that a sparse representation couldn't be used too.

\subsubsection{Treatment of singular $\gamma_d$ values}

For support vectors, the values of $\gamma_d$ will go to zero, or nearly zero.  Polson suggests using Greene's restricted least squares to separate support
vectors from non-SVs.  We found that
clamping the lambda values to be at least some small value $\epsilon$ gives similar results, and is simpler.

\subsubsection{Source-code}

Opensource code for the MPI and GPU implementations will be made available at http://ml.cs.tsinghua.edu.cn/{\mytilde}jun.

\subsection{MPI solver for linear classification}

\begin{table}
\centering
\begin{tabular}{|l|l|l|l|l|l|}
\hline
\multicolumn{5}{|l|}{ N=2,500,000 training subset:} \\
\hline
Solver   & P & C & Train & Acc. \% \\
\hline
Pegasos\cite{shalev2007pegasos}  & 1 & - & Crash & - \\
SDB\cite{chang2011selective} & 1 & 1 & Crash & - \\
StreamSVM\cite{matsushima2012linear} & 2 & 4e-5 & 6138s & 90.48 \\
SVMPerf\cite{joachims2006training} & 1 & 2 & 641.3 & 90.42 \\
LL-Primal\cite{Fan:jmlr08} & 1 & 4e-6 & 159.1s & 90.31 \\
LL-Dual\cite{Fan:jmlr08} & 1 & 4e-6 & 126.6s & 90.32 \\
LIN-EM-CLS & 48 & 1e-5 & 248.1s & 90.44 \\
LIN-EM-CLS & 480 & 1e-5 & \textbf{83.5s} & 90.45 \\
\hline
\multicolumn{5}{|l|}{ Full N=25,000,000 training set:} \\
\hline
Solver   & P & C & Train & Acc. \% \\
\hline
LL-Dual\cite{Fan:jmlr08} & 1 & 4e-6 & Crash & - \\
LL-Primal\cite{Fan:jmlr08} & 1 & 4e-6 & Crash & - \\
SVMPerf\cite{joachims2006training} & 1 & 2 & Crash & - \\
StreamSVM\cite{matsushima2012linear} & 2 & 4e-5 & > 30h & - \\
StreamSVM\cite{matsushima2012linear} & 2 & 1e-5 & 44h & 90.87 \\
LIN-EM-CLS & 48 & 1e-5 & 3327s & 90.81 \\
LIN-EM-CLS & 480 & 1e-5 & \textbf{533.8s} & 90.81 \\
\hline
\end{tabular}
\caption{Performance on dna dataset}
\label{tab:linemclsperfdna}
\end{table}

Table \ref{tab:linemclsperfdna} compares our LIN-EM-CLS implementation with other solvers for the dna dataset.

For a 2.5 million row subset, our solver was the fastest, when 48 cores were available. Other solvers tested were unavailable to take advantage of the extra cores.  Pegasos exceeded available memory (24GB + 30GB swap), and was killed.  SDB crashed for unknown reasons.

For the full 5 million row subset, our solver is one of the only two that managed to complete.  The 
other solvers exhausted available memory, and were killed.  StreamSVM makes good use of available memory, using a blocking procedure, but as a consequence, since it uses only two threads, and runs on a single network node, it is very slow.

Overall, our solver was the fastest on 2.5 million rows, for equivalent accuracy to other solvers, when 480 cores were available.  It was the only solver to complete training on the full 25 million rows within 24 hours, and was over three thousand times faster.  This might be explained partially because of the parallelization over multiple CPU cores, partly because the whole dataset can be loaded into distributed memory simultaneously, obviating any need for further I/O during training.

\subsection{Scalability of LIN-CLS}

\myfigure{Effect of number of cores on training speed, dna dataset}{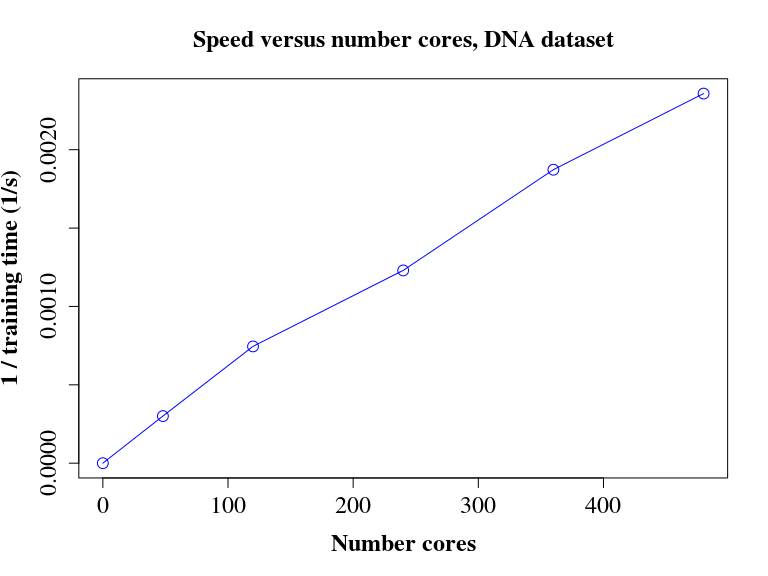}

Figure \myfigref{speedversusnumbercoresdnabig} shows the scalability with number of cores of LIN-EM-CLS, using the DNA dataset.  The speed is linear with the number of cores, as far as 480 cores, on this dataset.

\myfigure{Effect of $N$ on training time, alpha dataset}{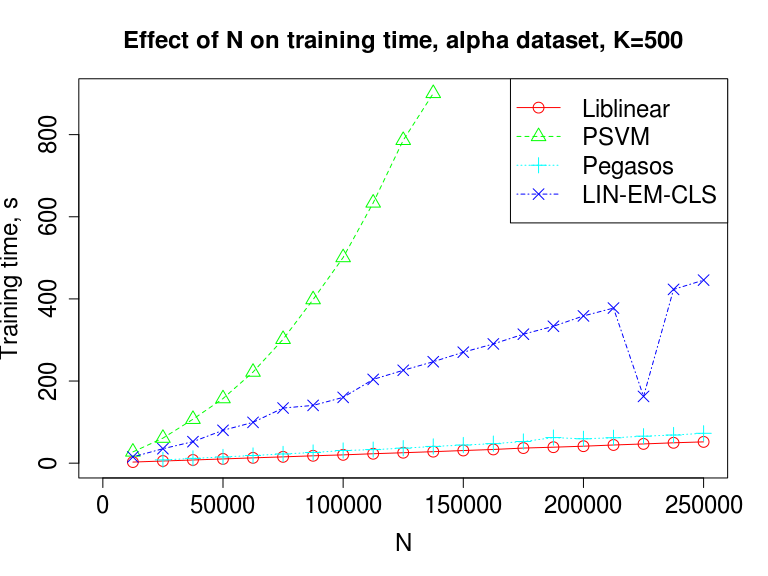}

Figure \myfigref{effectnontrainingtimealpha} shows the scalability with $N$. For this graph, all solvers were run single-threaded, including both LIN-CLS and PSVM. We can see that LIN-CLS is linear in $N$, and scales much better with $N$ than PSVM.  PSVM is a dual solver, and scales well with $K$, but less well with $N$.  Liblinear and Pegasos also scale linearly with $N$.  Note that LIN-EM-CLS is slower than Liblinear and Pegasos in a single-threaded scenario, but by taking advantage of additional cores, LIN-EM-CLS can be faster than both Liblinear and Pegasos.

\myfigure{Effect of $K$ on training time, alpha dataset}{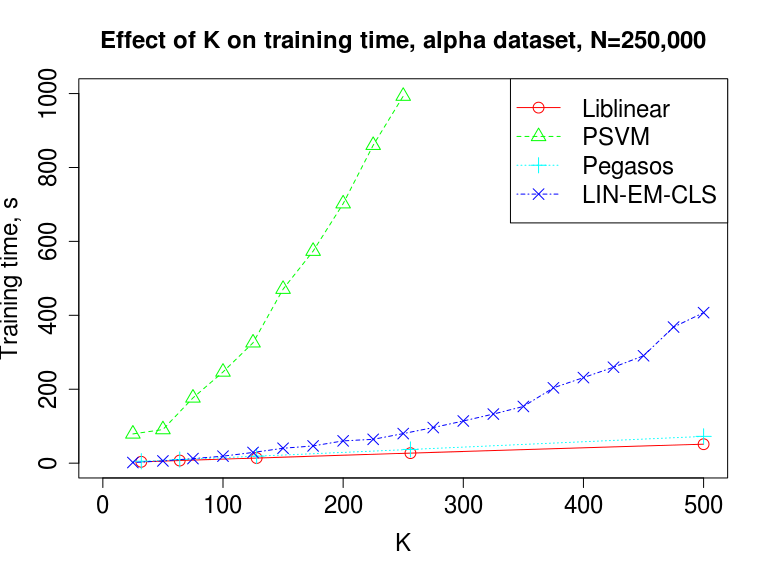}

Figure \myfigref{effectktrainingtimealpha} shows the effect of $K$ on training time, again running each solver single-threaded.  LIN-CLS is quadratic in $K$.  It scales better with $K$ than PSVM on this dataset.  This dataset is quite harsh on PSVM, because it has a very high $N$.  Liblinear and Pegasos are both linear with $K$.

\subsection{SVR}

\begin{table}
\centering
\begin{tabular}{|l|l|l|l|l|}
\hline
Solver & Cores & C & Train & RMS error \\
\hline
LL-Primal\cite{Fan:jmlr08} & 1 & 1 & 15.0s & 0.88 \\
LL-Dual\cite{Fan:jmlr08} & 1  & 1 & 114.9s & 0.89 \\
LIN-EM-SVR & 48 & 0.01 & \textbf{2.5s} & 0.90 \\
\hline
\end{tabular}
\caption{SVR on year dataset}
\label{tab:svryear}
\end{table}

Table \ref{tab:svryear} shows the performance of LIN-EM-SVR versus liblinear for the year regression dataset.  The data was normalized for mean and variance prior to testing.  Epsilon was set to 0.3.

LIN-EM-SVR trained the fastest, for similar accuracy.

\subsection{KRN}

\begin{table}
\centering
\begin{tabular}{|l|l|l|l|l|}
\hline
Solver & Cores & C & Train & Acc. \% \\
\hline
LL-Dual\cite{Fan:jmlr08} & 1 & 1000 & 7.1s & 90.2 \\
LL-Primal\cite{Fan:jmlr08} & 1 & 1000 & 1.67s & 90.3 \\
KRN-EM-CLS & 48 & 1 & 27.2s & 90.1 \\
\hline
\end{tabular}
\caption{KRN on N=1800 subset of news20}
\label{tab:krnnews20}
\end{table}

Table \ref{tab:krnnews20} shows results for KRN-EM-CLS.  Our accuracies are similar to liblinear for this training set.

The kernel formulation allows the use of non-linear kernels, and the training time is independent of $K$.

A limitation of the KRN formulation is that the training time is cubic in $N$.  

\subsection{Performance on Crammer and Singer models}

\begin{table}
\centering
\begin{tabular}{|l|l|l|l|l|l|}
\hline
Solver & Cores & $C$ & Train & Acc. \% \\
\hline
\multicolumn{5}{|l|}{N=200,000 training subset:} \\
\hline
LL-CS\cite{Fan:jmlr08} & 1 & 0.2 & 74.0s & 87.9 \\
SVMMult\cite{joachims2009cutting} & 1 & 800000 & 518.9s & 87.0 \\
LIN-MC-MLT & 48 & 0.04 & 284.4s & 86.1 \\
LIN-MC-MLT & 480 & 0.04 & 76.7s & 85.8 \\
\hline
\multicolumn{5}{|l|}{Full N=4,000,000 training set:} \\
\hline
SVMMult\cite{joachims2009cutting} & 1 & 80000 & Crash & - \\
LL-CS\cite{Fan:jmlr08} & 1 & 0.2 & 223.0s & 88.4 \\
LIN-MC-MLT & 48 & 0.04 & 4950s & 86.1 \\
LIN-MC-MLT & 480 & 0.04 & 613.9s & 86.3 \\
\hline
\end{tabular}
\caption{Performance on mnist8m dataset}
\label{tab:linemmltperfmnist8m}
\end{table}

Table \ref{tab:linemmltperfmnist8m} shows the performance of Crammer and Singer classifiers on the mnist8m dataset.  Our implementation of Crammer and Singer is parallelizable across multiple cores.  On a cluster today it gives training times comparable to liblinear,
and much faster than SVMMulticlass.  In the future, the number of cores available will likely increase, possibly exponentially, and our implementation
might become increasingly advantageous, when compared to the single-threaded liblinear and SVMMulticlass.

For the full mnist8m dataset, only our solver and liblinear were able to complete training.  SVMMulticlass used up all available memory (24GB + 30GB swap), and was killed. 

Increasing the number of cores from 48 to 480 for our implementation gave a 7.6 times increase in speed, showing the scalability of this algorithm.

\subsection{Convergence}

\myfigure{Convergence of objective function, DNA, N=2.5 million subset, C=1e-5}{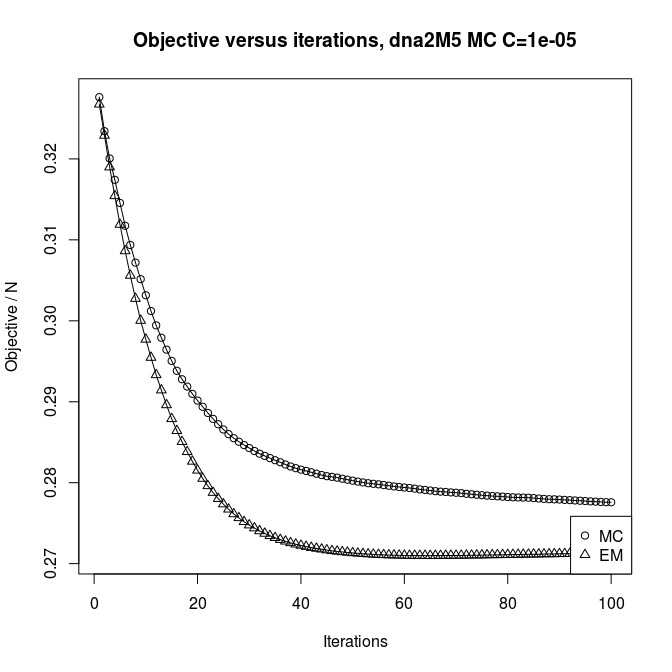}

\myfigure{Convergence of accuracy, DNA, N=2.5 million subset, C=1e-5}{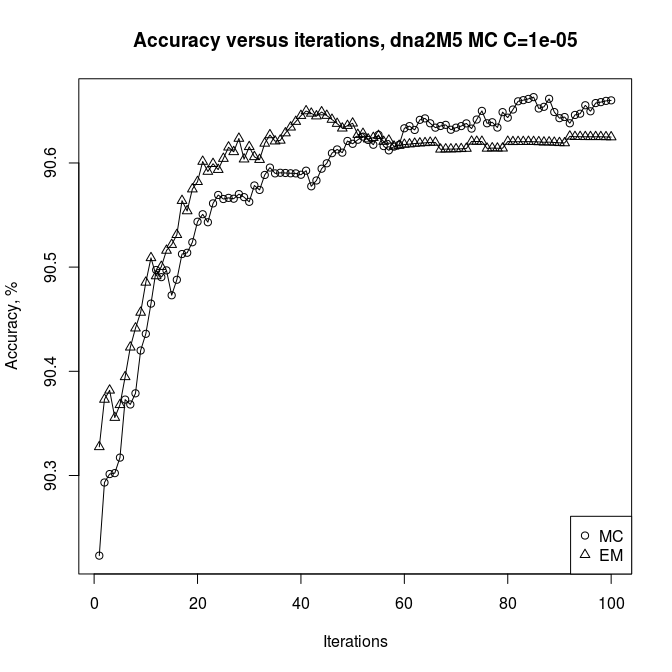}

Figure \myfigref{objectiveversusiterationsdna2m5} shows the convergence of the objective function, for both
MC and EM, for the DNA dataset, for LIN-*-CLS.

The EM objective function converges within 40-60 iterations here, and this is what we saw in practice
across other datasets.

For MC, we have two choices:
\myitemize{
     \item use the best single sample
     \item average multiple samples
}

Given that this is high-dimensional space, taking single samples is unlikely to get close to the optimal solution, so we average across multiple samples.  Usually, one would want to select a small burnin period of 10-20 iterations.

This contrasts with EM, where we use a single sample at each iteration to measure the test accuracy.

So, the objective function for MC in these graphs converges more slowly than for EM.

In this graphs, we didn't use a burnin period for MC.  Using a burnin period of 10 iterations improves the convergence time.

Taking the average across all MC samples from $1$ to $i$ gives a relatively smooth change in the objective function over time, which is useful for making a convenient stopping heuristic, and also gives good test accuracy..

Note that whilst this particular dataset gives a monotonically decreasing objective function for MC, we noticed that in some cases the objective convergence curve does have multiple local minima, so one needs to be a bit careful as to how to construct an appropriate stopping heuristic.

Figure \myfigref{accuracyversusiterationsdna2m5} shows the accuracy, for the same experiment.

We can see that whilst EM converged faster in this case to a solution, after 100 iterations, the test accuracy for MC was higher.

In practice, we found that for LIN-*-CLS, EM gave good accuracies, and the stopping heuristics are simpler.

For the Crammer and Singer implementation, MC converged much faster than EM.

\subsection{Parallelization using GPU}

GPU kernels were written to evaluate the $\Sigmav$ component of the algorithm, ie $\sum_d \frac{1}{\gammav_d}\xv_d \xv_d^T$.  This is the rate-limiting
step for many datasets.  For LIN, the execution time is asymptotically $O(N K^2)$.

\begin{table}
\centering
\begin{tabular}{|l|l|l|}
\hline
Implementation & Time & Relative speed \\
\hline
1 CPU core & 17.1s & 1 \\
512 GPU cores & 0.73s & 23 \\
2048 GPU cores & \textbf{0.34s} & \textbf{50} \\
\hline
\end{tabular}
\caption{Using GPU to evaluate $\Sigmav$, for $N=250,000$, $K=500$}
\label{tab:sigmagpu}
\end{table}

Table \ref{tab:sigmagpu} shows the results for evaluating $\Sigmav$ for simulated $\xv_d$ and $\gammav_d$ vectors.
Using 512 GPU cores was 23 times faster than a single core.  Using 2048 GPU cores was about 50 times
faster.  The CPU core was from an Intel i7-3930K 3.2GHz CPU, and the GPU cores were from nVidia GTX590 GPUs (one GPU contains 512 cores).

\begin{table}
\centering
\begin{tabular}{|l|l|l|l|l|l|}
\hline
Solver & Hardware & Data & Learn & Acc. \% \\
 &  & load &  & \% \\
\hline
LL-Dual\cite{Fan:jmlr08}  & 1 CPU core & \multicolumn{2}{|c|}{ \hspace{12pt} 44.8s} & 78.16 \\
LIN-EM-CLS & 1 CPU core & 30.4s & 78.9s & 75.4 \\
LIN-EM-CLS & 2048 GPU cores & 29.2s & \textbf{6.1s} & 75.4 \\
\hline
\end{tabular}
\caption{GPU performance on alpha dataset, C=1}
\label{tab:gpualpha}
\end{table}

Table \ref{tab:gpualpha} shows the performance of LIN-EM-CLS using a GPU implementation on the alpha dataset.

We can see that using a single CPU core, liblinear is nearly 3 times faster than LIN-EM-CLS.  However, using GPU cores accelerated the learning time For LIN-EM-CLS by 13 times, relative to the single CPU core version.

Note that for this dataset, the data load time dominates the GPU version. This is the time to load the data from storage into PC main memory.  This is limited (i) by the speed of the storage medium and (ii) by the speed of parsing the ASCII data using a single CPU core.

One
advantage of the MPI implementation over the GPU version is that I/O is parallelized over multiple processors and multiple compute nodes.


\section{Conclusions}\label{section:conclusions}

We have presented a simple technique to solve SVM models on parallel hardware, using a
sampling SVM.  Our implementation of a parallel linear SVM solver is capable of handling very large
datasets, and scaled up to at least several hundred cores in our experiments.  We have provided
an extension to non-linear kernels, an implementation for support vector regression, and
a parallel solver for the Crammer and Singer model.  It is useful in its own right, and it 
is a useful addition to our armory, enabling fast and accurate solutions to composite maximum-margin models.

\bibliographystyle{abbrv}
\bibliography{psvm}

\end{document}